\documentclass{article}

\usepackage[preprint]{neurips_2026}
\usepackage{macro}
\usepackage[utf8]{inputenc} 
\usepackage[T1]{fontenc}    
\usepackage{hyperref}       
\usepackage{url}            
\usepackage{booktabs}       
\usepackage{amsfonts}       
\usepackage{nicefrac}       
\usepackage{microtype}      
\usepackage{xcolor}         

\title{Masking Causality and Conditional Dependence}

\author{%
  Zou Yang\\
  Thayer School of Engineering\\
  Dartmouth College\\
  Hanover, NH, 03755 \\
  \texttt{zou.yang.th@dartmouth.edu} \\
  \And
  Sophia Xiao\\
  Thayer School of Engineering\\
  Dartmouth College\\
  Hanover, NH, 03755 \\
  \texttt{sophia.xiao.th@dartmouth.edu} \\
  \AND
  Bijan H. S. Mazaheri\thanks{\href{http://bijanmazaheri.com}{bijanmazaheri.com}, \href{cedarlab.dartmouth.edu}{cedarlab.dartmouth.edu}} \\
  Thayer School of Engineering\\
  Dartmouth College\\
  Hanover, NH, 03755 \\
  \texttt{bijan@dartmouth.edu} \\
}

\begin{document}

\maketitle

\begin{abstract}
  Many regulatory and analytic problems require that a prohibited
variable influence a decision only through a designated allowable
channel --- a conditional-independence requirement that arises in
path-specific fairness, the handling of classified information, and
the regulation of trading on non-public information, among other
settings. Such requirements may be enforced either stratum-by-stratum
or, more commonly (and more efficiently), through a single averaged
constraint on the conditional effect. We study the resulting
enforcement problem from two perspectives. From the regulator's
side, we formulate causal masking as a linear program and show that
averaged-constraint optimization almost surely produces policies that
violate the stratum-wise requirement while satisfying the averaged
one exactly. The gains from masking grow with confounding and
outcome heterogeneity, and detection requires precisely the
conditional-independence tests that average constraints aim to
avoid. From the optimizer's side, the same construction shows that
masked policies recover most of the reward of unconstrained
exploitation while being far harder to detect, making them attractive
in any setting where the basis of decisions is itself sensitive.
Together, these results argue that regulating direct dependence
through averaged statistics on observed decisions is structurally
limited, and that meaningful enforcement must operate at the level
of the decision rule itself.
\end{abstract}

\section{Introduction}\label{sec:intro}

Many regulatory and analytic settings share a common structure: a decision-maker is allowed to use some information but not other information, and dependence between the prohibited variable and the decision is permitted only \emph{through} a designated allowable channel. Statistically, this is a conditional-independence requirement --- the prohibited variable should be independent of the decision given the allowable mediator. What is forbidden is direct dependence; what is desirable, often, is dependence through the allowable path.

This structure arises across very different domains. In path-specific fairness \citep{nabi_fair_2018}, a protected attribute may legitimately influence a decision through mediators such as qualifications, or even counfounders, but not through pathways representing direct discrimination. In the handling of classified intelligence, an analyst's public statements or trades may reflect open-source reasoning but must not betray knowledge of a compromised secret. In financial markets, a trader's positions may be driven by lawful quantitative analysis but must not reflect material non-public information. In each case, the regulator's task is the same: from observed decisions, distinguish dependence that flows through the allowable channel from dependence that does not.

We study this problem from two complementary perspectives. The \emph{auditor} observes decisions (and the variables conditioned upon) and seeks to detect direct dependence that cannot be attributed to the allowable channel. The \emph{optimizer} chooses a decision policy to maximize a reward while remaining indistinguishable, from the auditor's vantage point, from a policy that respects the conditional-independence requirement. Although we phrase the problem adversarially, our results apply equally to non-adversarial settings: any reward-driven optimizer constrained only on averages will, generically, gravitate toward the same solutions an adversarial agent would deliberately construct.

\subsection{Two Forms of Regulation}

A regulator wishing to enforce conditional independence has two natural choices. The strict choice requires independence to hold \emph{stratum by stratum}: the conditional effect of the prohibited variable on the decision must vanish at every value of the allowable mediator. The averaged choice requires only that the \emph{average} conditional effect vanish, with the average taken over the distribution of the mediator. The averaged form is the one adopted in path-specific fairness optimization by \citet{nabi_learning_2019}, and it is attractive for good reasons: it is statistically tractable, integrates cleanly into optimization as a single linear constraint, and sidesteps the known information-theoretic hardness of conditional independence testing \citep{shah_hardness_2020, he_hardness_2025}.

We argue that the averaged form is fit for one purpose but not for another. As a \emph{detector} of unregulated behavior, an averaged conditional effect is informative: a decision-maker who makes no effort to disguise their use of the prohibited variable will typically register a non-zero average. As a \emph{regulation}, however, the averaged form is brittle. An optimizer can satisfy an averaged constraint exactly while recovering most of the reward available under no constraint at all, by trading positive direct effects in one stratum against negative direct effects in another. Outcomes are typically heterogeneous across strata, and a constraint that restricts only the average leaves this heterogeneity available for exploitation. We call policies that exploit it \emph{causally masked}.

The principal warning of this paper is not that an averaged constraint \emph{can} be escaped, but that it \emph{will} be escaped whenever the constraint binds, and there is a reward at stake. It is not obvious, ex ante, that treating subgroups differently should help maximize an objective that itself averages over those subgroups; the very framing of the constraint can lull both regulator and optimizer into believing that compliance and reward are aligned. Our linear-programming formulation makes the mechanism transparent and shows that the gains from masking are typically substantial.

\subsection{Summary of Contributions}

Our contributions speak to both sides of the regulatory relationship.

\paragraph{Regulator's side.} \Cref{sec: LP formulation} formulates causal masking as a linear program that maximizes an auxiliary reward subject to an averaged conditional-effect constraint, and compares its solutions to those of stratified-constraint and unconstrained policies. We establish theoretical results (\Cref{sec: theory}) showing that the utility gap between a stratum-wise constrained policy and a masked policy grows with confounding between the prohibited variable and the mediator, and with outcome heterogeneity across the mediator. Through genericity arguments (Appendix~\ref{apx: measure theory}), we show that averaged-constraint optimization almost surely produces masked, rather than stratum-wise compliant, policies --- whether or not the optimizer is adversarial. We further argue that distinguishing masked policies from compliant ones using observed decisions alone requires precisely the conditional-independence tests that other approaches aimed to avoid, so masked solutions are hard to detect and likely to persist. We illustrate these phenomena empirically on the COMPAS dataset \citep{angwin_machine_2022} in \Cref{sec: empirical}.\footnote{Code for both real-data and synthetic experiments is available at an \href{https://anonymous.4open.science/r/causal_masking-F908/README.md}{anonymized GitHub link} and was run in $<$2 hours on a regular CPU.}

\paragraph{Optimizer's side.} The same setup has constructive value for any agent operating under an averaged constraint who wishes to act on prohibited information. Relative to a stratum-wise constrained policy, masked policies recover most of the reward available under unconstrained exploitation; relative to a pure-exploitation policy, masked policies are dramatically harder to detect, since an averaged audit cannot distinguish them from compliant policies and a conditional-independence audit is information-theoretically hard \citep{shah_hardness_2020}. This trade-off is relevant in any setting where the basis of one's decisions is itself sensitive: a trader concealing which signals drive a strategy, an analyst protecting a source, or --- more concerningly --- an automated decision system rewarded for performance and constrained only on averages.

Taken together, these results imply that regulation of direct dependence based on averaged statistics over observed decisions is structurally limited. Meaningful enforcement requires \emph{model-level}, in-process mechanisms \citep{pessach_review_2022} that constrain the decision rule itself, rather than reactive audits of its outputs. For fairness, this corresponds to notions of counterfactual fairness \citep{kusner_counterfactual_2017} and path-specific counterfactual fairness \citep{chiappa_path-specific_2019}.

\subsection{Related Works} \label{subsec: related works}

\paragraph{Formalizing Fairness.}
Many works define fairness as an independence property between
protected attributes and decisions \citep{pessach_review_2022}.
Statistical parity requires the protected attribute and decision to
be marginally independent \citep{dwork2012fairness}; conditional
statistical parity relaxes this to parity within subgroups
\citep{corbett-davies_algorithmic_2017}. \citet{nabi_fair_2018}
introduce path-specific fairness via causal-effect constraints that
distinguish permissible from impermissible pathways, and
\citet{kusner_counterfactual_2017} propose a counterfactual notion
at the individual level via structural causal models.
\citet{chiappa_path-specific_2019} extends the path-specific
framework to a counterfactual formulation, and
\citet{plecko_causal_2024} provide a comprehensive causal toolkit
for fair machine learning. These approaches address fairness from
the modeler's side; our work is complementary, characterizing what
a regulator can enforce using only observed decisions.

\paragraph{Optimization and Fairness.}
\citet{dwork2012fairness} first formalized fairness as a constraint
within an optimization problem over individual decisions.
\citet{corbett-davies_algorithmic_2017} extended this to quantify the
``cost of fairness'' in the objective, and
\citet{corbett-davies_measure_2023} studied the geometry of fairness
constraints, showing they often produce suboptimal solutions that
harm the groups they intend to protect. Our optimization differs in
focusing on decision rates across groups rather than individual
decisions, and in studying the consequences of \emph{averaged}
effect-based constraints specifically. \citet{nabi_learning_2019}
present an alternative policy-optimization formulation for online
reinforcement learning that uses the path-specific effect constraints
of \citet{nabi_fair_2018}; those constraints are also averaged, so
our results on the limits of averaged regulation apply to that
setting as well.

\paragraph{Fairness in Predictions and Counterfactual Harm.}
A parallel literature studies fairness as a property of predictions
evaluated against ground-truth outcomes --- including equalized odds
\citep{barocas_fairness_2019}, equal opportunity
\citep{hardt_equality_2016}, predictive parity
\citep{chouldechova_fair_2017}, calibration
\citep{pleiss_fairness_2017}, and principal fairness
\citep{imai_principal_2023}. Our setting differs in that the
regulator observes decisions only, with no separate outcome label.
A separate line of work formalizes individual-level harm through
counterfactuals \citep{richens_counterfactual_2022,
kallus_whats_2022}, bounding or optimizing counterfactual harm assuming the policy or structural causal model is fixed --- whereas we study how an optimizer responds to an averaged
regulatory constraint and whether the resulting policy is detectable from observed decisions alone.

\paragraph{Information Leakage Games.}
A structurally similar two-player problem appears in quantitative
information flow, where leakage of a secret to an adversary is
formalized via gain-function--based measures \citep{smith_foundations_2009} and as a game between a defender choosing a channel and an adversary estimating the secret \citep{alvim_information_2022}. Closer still in form is information-theoretic steganography
\citep{cachin_information-theoretic_2000}, where a sender embeds a
message in a cover so that the stego distribution is statistically
indistinguishable from the cover; masked policies are a direct
analogue, embedding dependence on the prohibited variable while
remaining indistinguishable from compliant policies under the
auditor's averaged test. Our results echo the security
literature's emphasis that detection must be calibrated to an
adaptive sender. Whether randomization-based defenses from this literature are sufficient to mask direct dependence is an empirical question; we
return to it in \Cref{sec:discussion}.

\textbf{Detecting Heterogeneous Effects.} Detecting causal masking
from observed decisions reduces to a heterogeneity test on the
conditional effect. Methods for estimating
heterogeneous treatment effects offer one route
\citep{wager_estimation_2018, mazaheri_synthetic_2025}, but
hypothesis testing suffices for detection, and existing heterogeneity
tests \citep{ding_randomization_2016, chung_permutation_2021} apply
directly when covaraites are continous. These hypothesis tests will be essential for assessing
fairness when the underlying decision mechanism is unknown, as is
generally the case for human decisions or proprietary models. They
also provide a natural starting point for extending our diagnostic
framework to settings with continuous covariates, where binning
introduces researcher degrees of freedom.

\section{Preliminaries}

\paragraph{Notation.}
We use capital Roman letters to denote random variables and the corresponding lowercase letters to denote their realizations. We write $x_i$ as shorthand for $X = i$. Lists of random variables are bolded, e.g., $\bvec{X}$, and parameter vectors are denoted with vector notation, e.g., $\vec\alpha$.

\paragraph{Causal Graphs.}
We describe our settings using \emph{structural causal models (SCMs)} \citep{pearl2009causality}, in which directed edges between random variables (vertices) represent direct causal relationships. A central insight of \citet{pearl2009causality} is that the causal structure determines the statistical independencies of the variables via the graphical d-separation rules. The faithfulness assumption, common in causal-discovery literature \citep{vowels2022d}, asserts that d-connected variables are statistically dependent. Faithfulness holds with full Lebesgue measure \citep{meek_strong_1995}, but \citet{uhler2013geometry} showed that the set of distributions $\varepsilon$-close to a faithfulness violation grows quickly with $\varepsilon$. Finite-sample conditional-independence tests are therefore much more vulnerable to ``near'' faithfulness violations than asymptotic guarantees suggest --- a point we return to when discussing relaxations of the regulatory constraint.

\paragraph{Causal Effects.}
We study the causal effect of a prohibited variable $P$ on a decision $D$, with the goal of regulating this effect to flow only through an allowable mediator $X$. The natural target of regulation is a \emph{path-specific effect} (PSE) \citep{pearl2009causality, nabi_fair_2018}: the contribution to the effect of $P$ on $D$ that does \emph{not} flow through $X$. Under the standard adjustment formula
\begin{equation} \label{eq: ate def}
    ATE \coloneqq \sum_{x} \Pr(x) \big[\E(D \given x, p_1) - \E(D \given x, p_0)\big],
\end{equation}
this PSE coincides with the average treatment effect (ATE) of $P$ on $D$ when $X$ is a backdoor adjustment set (e.g., $X$ contains the confounders of $P$ and $D$), and with the controlled direct effect when $X$ is a mediator. In either case, requiring this quantity to vanish forbids direct dependence of $D$ on $P$ while permitting dependence routed through $X$. For convenience, we will use the term ATE to refer to this quantity.

The terms inside the sum in Eq.~\eqref{eq: ate def} are the \emph{conditional} average treatment effects (CATEs): each measures the change in $\E(D)$ within the stratum $X = x$ when $P$ is varied. A \emph{stratified} regulation requires every CATE to vanish; an \emph{averaged} regulation requires only that the ATE itself vanish. The gap between these two requirements is the central object of this paper.

The propensity score $\Pr(p \given x)$ --- the probability that an individual belongs to group $P = p$ given covariates $X = x$ --- governs the relative weight of CATEs in the ATE, and will appear in our analysis of masking gains in \Cref{sec: theory}.

\section{Motivating Example} \label{sec: motivating example}

We illustrate how an optimizer can satisfy an averaged constraint exactly while extracting reward from heterogeneous treatment of subgroups. The example is drawn from the fairness setting because the masking mechanism is most transparent there; \Cref{sec: empirical} demonstrates the same phenomenon on the COMPAS dataset, and we discuss analogous structures in trading and intelligence settings at the end of this section.

\paragraph{Setup.}

A university uses an admissions model ($D = 1$ corresponds to admission) for two departments, $X \in \{0,1\}$, and is required to be ``fair'' with respect to a protected attribute $P \in \{0,1\}$. The structure mirrors the Berkeley admissions case \citep{bickel_sex_1975}; for clarity we use the deliberately simple distribution in \Cref{tab: synth crime data}.

\begin{table}[h]
    \centering
    \scalebox{.75}{\begin{tabular}{|c|c|c|c|c|}
    \hline
    $x$ & $p$ & $\Pr(x,p)$ & $\Pr(Y=1 \given x, p)$ & $\Pr(D_{\text{mask}}=1\mid x,p)$\\
    \hline
    $0$ & $0$ & $\nicefrac{1}{15}$ & $.5$ & $.5$\\
    \hline
    $0$ & $1$ & $\nicefrac{9}{15}$ & $.5$ & $0$\\
    \hline
    $1$ & $0$ & $\nicefrac{4}{15}$ & $.25$ & $0$\\
    \hline
    $1$ & $1$ & $\nicefrac{1}{15}$ & $1$ & $1$\\
    \hline
\end{tabular}}
    \caption{A joint distribution on $(X, P)$, the conditional graduation rate $\Pr(Y=1 \mid x, p)$, and a masking policy $D_{\text{mask}}$.}
    \label{tab: synth crime data}
\end{table}

The school's objective is graduations among admitted students ($Y \in \{0,1\}$). The protected attribute $P$ may not influence $D$ directly, but may legitimately influence $D$ through the department $X$ --- e.g., because $P$ correlates with which department applicants apply to and departments differ in selectivity.

\begin{figure}[h]
    \centering
    \scalebox{.65}{\begin{tikzpicture}[
    -latex, auto, 
    node distance = 1.5cm and 1.5cm, 
    on grid, 
    ultra thick,
    state/.style={circle, draw, minimum width=.5cm, ultra thick, fill=white}
]
    \begin{scope}[shift={(0,0)}]
        \node[state] (P-a) {$P$};
        \node[state] (X-a) [below = of P-a] {$X$};
        \node[state] (D-a) [left = of X-a] {$D$};
        \node[state] (Y-a) [right = of X-a] {$Y$};
        
        \draw[arrows=-] (X-a) -- (Y-a);
        \draw[arrows=-] (X-a) -- (P-a);
        \draw[arrows=-] (P-a) -- (Y-a);
        \path (X-a) edge (D-a);
    \end{scope}

    \begin{scope}[shift={(5.5cm,0)}]
        \node[state] (P-c) {$P$};
        \node[state] (X-c) [below = of P-c] {$X$};
        \node[state] (D-c) [left = of X-c] {$D$};
        \node[state] (Y-c) [right = of X-c] {$Y$};
        \draw[arrows=-] (X-c) -- (P-c);
        \draw[arrows=-] (P-c) -- (Y-c);
        \path (P-c) edge[red!60!black] (D-c);
        \path (X-c) edge (D-c);
        \draw[arrows=-] (X-c) -- (Y-c);
    \end{scope}

    \begin{scope}[on background layer]
        \coordinate (top-stretch-ab) at ([yshift=.8cm]P-a.north);
        \coordinate (top-stretch-c)  at ([yshift=.8cm]P-c.north);

        \node[
            fit=(P-a)(D-a)(Y-a)(top-stretch-ab), 
            rounded corners, 
            fill=blue!10, 
            draw=blue!60!black,
            line width=2pt,
            inner xsep=10pt,
            inner ysep=5pt
        ] (blue-box) {};
        \node[text=blue!60!black, anchor=north, yshift=-5pt] at (blue-box.north) 
            {\Large \textbf{Compliant}};

        \node[
            fit=(P-c)(D-c)(Y-c)(top-stretch-c), 
            rounded corners, 
            fill=red!10, 
            draw=red!70!black,
            line width=2pt,
            inner xsep=20pt,
            inner ysep=5pt
        ] (red-box) {};
        \node[text=red!70!black, anchor=north, yshift=-5pt] at (red-box.north) 
            {\Large \textbf{Masked / Exploitative}};
    \end{scope}
\end{tikzpicture}}
    \caption{Causal diagrams for the two policy classes. Undirected edges indicate flexibility in causal direction and the possibility of unobserved confounding. The directed edge $X \rightarrow D$ represents the allowable use of $X$ by the decision rule; the red edge $P \rightarrow D$ in the right panel is the prohibited direct dependence introduced by exploitative and masked policies.}
    \label{fig:causal diagrams}
\end{figure}

\Cref{fig:causal diagrams} gives causal diagrams for the two policy classes. The left (compliant) diagram has no $P \rightarrow D$ edge, so $P \indep D \given X$. The right (masked/exploitative) diagram introduces the red $P \rightarrow D$ edge: the decision depends directly on $P$. Distinguishing these graphs from observed decisions alone is exactly the auditor's task.

\paragraph{Stratified versus Averaged Constraints.}

Suppose a regulator wishes to enforce $P \indep D \given X$.The strict, stratified version of this requirement is that each CATE must vanish:
\begin{align*}
    \E(D \given x_0, p_1)  - \E(D \given x_0, p_0) &= 0, & \E(D \given x_1, p_1) - \E(D \given x_1, p_0) &= 0.
\end{align*}
The averaged version is much milder:
\begin{align*}
    ATE &= \Pr(x_0)\big(\E(D \given x_0, p_1)  - \E(D \given x_0, p_0)\big) + \Pr(x_1)\big(\E(D \given x_1, p_1) - \E(D \given x_1, p_0)\big) = 0.
\end{align*}
The two are equivalent only if each CATE happens to vanish individually; otherwise, the averaged constraint admits CATEs of opposite sign that cancel.

\paragraph{Evading the Averaged Constraint.}

Every policy satisfying the stratified constraint admits at most $50\%$ of its students to graduation, since both fair admission rules treat $X = 0$ and $X = 1$ identically with respect to $P$. An optimizer willing to use $P$ openly can do better by admitting only $P = 1$ applicants in the second department ($X = 1$), achieving a perfect graduation rate. This exploitative strategy registers an ATE of $1/3$ and is flagged by either form of regulation.

The masking strategy $D_{\text{mask}}$ in \Cref{tab: synth crime data} is more subtle. It admits $P = 0$ applicants in $X = 0$ and $P = 1$ applicants in $X = 1$, rejecting the other two strata, achieving an ATE of exactly zero by offsetting positive direct effects in $X = 0$ against negative direct effects in $X = 1$. The policy is unfair within each stratum --- it disfavors the $P = 0$ group when $X = 1$ and the $P = 1$ group when $X = 0$ --- but it satisfies the averaged constraint exactly, and it raises the graduation rate from $\nicefrac{1}{2}$ (compliant) to $\nicefrac{5}{6}$. An optimizer responsive to the graduation reward and constrained only on average is therefore likely to converge on a policy of this form.

\paragraph{Other Instantiations.}

The same structure arises whenever a prohibited variable is regulated through an averaged effect on observed decisions. A parole board may need to show that release decisions ($D$) are free of bias by a protected attribute ($P$) when controlling for case details ($X$); we study this case empirically in \Cref{sec: empirical}. A trading desk may need to certify that its trades ($D$) are independent of material non-public information ($P$) given lawful quantitative signals ($X$). An analyst's public statements ($D$) may need to be independent of classified information ($P$) given the open-source record ($X$). In every case, an averaged constraint on the effect of $P$ on $D$ leaves the same masking maneuver available.

\section{Causal Masking} \label{sec: LP formulation}

We formalize the masking problem as a linear program. Let
$D \in \{0, 1\}$ be a decision representing ``when to participate''
--- granting parole, making a hire, admitting a student, executing
a trade. The goal is to maximize a reward $Y$ when $D = 1$, i.e.,
$\mathbb{E}[Y \mid d_1]$, equivalently $\mathbb{E}[Y \mid d_1]
\Pr(d_1)$ subject to a fixed participation rate $\Pr(d_1) = \rho$.

Features are divided into a prohibited variable $P \in \{0, 1\}$
and an allowable mediator $X \in \{1, \ldots, k\}$. We treat both
as discrete; continuous or multivariate $X$ can be reduced to this
setting by clustering, with $k$ a natural resolution parameter. 
A decision policy is a function $\alpha(x, p) \in [0, 1]$, with
$D \mid x, p \sim \text{Bernoulli}(\alpha(x, p))$, and we write
$\gamma_{x,p} = \mathbb{E}[Y \mid x, p]$ and
$\pi_{x,p} = \Pr(x, p)$.

\subsection{LP Formulation}

Assuming $D \indep Y \mid X, P$, the expected reward of a policy is
$\mathcal{W}(D) = \sum_{x,p} \gamma_{x,p}\, \alpha_{x,p}\, \pi_{x,p}$,
and the unconstrained optimization, subject to a constant participation rate $\rho$, is
\begin{equation} \label{eq: base lp}
    \max_{\{\alpha_{x,p}\}} \mathcal{W}(D)
    \quad \text{s.t.} \quad
    0 \le \alpha_{x,p} \le 1 \;\; \forall x,p,
    \;\;
    \sum_{x,p} \alpha_{x,p}\, \pi_{x,p} = \rho.
\end{equation}

\begin{proposition}[Optimal Exploit Policy] \label{prop: opt exploit}
The \emph{exploit} policy solves Eq.~\eqref{eq: base lp}, using $X$
and $P$ without restriction. Its optimum participates greedily in
$(x, p)$ pairs in decreasing order of $\gamma_{x,p}$ until the
participation rate $\rho$ is exhausted.
\end{proposition}

The two regulated policies are obtained by adding a single
constraint to Eq.~\eqref{eq: base lp}.

\begin{proposition}[Optimal $\varepsilon$-Unaware Policy]
\label{prop: opt unaware}
The \emph{$\varepsilon_{\text{una}}$-unaware} policy
$D_{\text{una}}(\varepsilon_{\text{una}})$ enforces approximate
stratified independence $P \indep D \mid X$ via
\begin{equation*}
    -\varepsilon_{\text{una}} \leq \alpha_{x,1} - \alpha_{x,0}
    \leq \varepsilon_{\text{una}} \quad \forall x = 1, \ldots, k.
\end{equation*}
At $\varepsilon_{\text{una}} = 0$, define $w_x^{\text{avg}}
\coloneqq \E[Y \mid x] = \gamma_{x,1}\Pr(p_1 \mid x) +
\gamma_{x,0}\Pr(p_0 \mid x)$. The optimal policy participates in
strata in decreasing order of $w_x^{\text{avg}}$, treating both
$P$-groups within a stratum identically.
\end{proposition}

\begin{proposition}[Optimal $\varepsilon$-Masking Policy]
\label{prop:mask}
The \emph{$\varepsilon_{\text{mask}}$-masking} policy
$D_{\text{mask}}(\varepsilon_{\text{mask}})$ relaxes the unaware
constraint to bound only the ATE:
\begin{equation*}
    -\varepsilon_{\text{mask}} \leq \sum_{x} (\pi_{x, 0} + \pi_{x, 1})
    (\alpha_{x,1} - \alpha_{x,0}) \leq \varepsilon_{\text{mask}}.
\end{equation*}
\end{proposition}

We will usually consider $\varepsilon_{\text{una}} =
\varepsilon_{\text{mask}} = 0$, yielding the strict policies
$D_{\text{una}}$ and $D_{\text{mask}}$; nonzero $\varepsilon$
Relaxations are used to study finite-sample behavior in
Appendix~\ref{apx: synthetic}. Since the unawareness constraint implies the
masking constraint, the optimal masking policy is always \emph{at
least as good} as the optimal unaware policy.

\section{Theoretical Results} \label{sec: theory}

This section establishes key properties of causal masking. We
lower-bound the benefit of masking relative to the unaware solution
and argue that masking is a \emph{natural outcome} of optimization
under an averaged constraint, even when the optimizer is not
adversarial. All proofs are in Appendix~\ref{apx: deferred proofs}.

\subsection{Improvements from Masking}

The optimal unaware policy $D_{\text{una}}$ is a ``water-filling''
algorithm that selects strata in decreasing order of $w_x^{\text{avg}}$
until the participation budget $\rho$ is met. If $\rho$ is small
enough, this reduces to choosing the single best stratum
$x^* = \arg\max_x w_x^{\text{avg}}$. From this baseline, an optimizer
can swap participation between groups in different strata to improve
utility while preserving an averaged constraint, scaling rates to
keep $\rho$ fixed. This swap-and-scale construction yields a lower
bound on the masking advantage in the small-$\rho$ regime.

\begin{definition}[Arbitrage Rate of Return]
For a policy that trades participation between states
$(i, \overline{p})$ and $(j, p)$, define
\[ R_p(j) \coloneqq
   \frac{w_{i,\overline{p}} + w_{j,p}}{\Pr(\overline{p}\mid x_i) +
   \Pr(p\mid x_j)}, \]
where $\overline{p} = |1 - p|$ and
$w_{i,p} = \gamma_{x_i,p}\Pr(p\mid x_i)$.
\end{definition}

\begin{theorem}[The Masking Gap]
\label{thm: fair vs mask corrected}
Suppose $\rho$ is small enough that $D_{\text{una}}$ participates
only in the best stratum $x^* = x_i$. The performance gap
$\Delta = \mathcal{W}(D_{\text{mask}}) - \mathcal{W}(D_{\text{una}})$
satisfies
\[ \Delta \ge \rho \cdot \max\!\left(0,\;
   \max_{j \ne i,\, p \in \{0,1\}}
   \{ R_p(j) - w_i^{\text{avg}} \}\right). \]
\end{theorem}
Applied to the motivating example in Table~\ref{tab: synth crime data},
this bound gives $\Delta \geq \rho/3$, which matches the actual
gap of $\nicefrac{1}{30}$ at $\rho = 0.1$.

\subsection{The Role of $X$}

A nonzero gap requires the optimizer to exploit either (i) demographic
disparity between strata or (ii) heterogeneous outcomes across
strata --- equivalently, $P \notindep X$ (confounding) or
$X \notindep Y \mid P$ (heterogeneity). Each is sufficient on its
own, by inspection of the arbitrage rate $R_p(j)$ relative to the
unaware return $w_i^{\text{avg}}$.

\textbf{Heterogeneity ($X \notindep Y \mid P$).} Even with no
confounding, $w_{x,p} = \gamma_{x,p}\Pr(P = p)$ varies across strata,
and $R_p(j)$ becomes a mix of reward terms that can outperform
$w_i^{\text{avg}}$.

\textbf{Confounding ($P \notindep X$).} Even with no heterogeneity
($\gamma_{x, p} = \gamma_p$), $R_p(j)$ and $w_i^{\text{avg}}$ are
weighted averages of $\gamma_p$ and $\gamma_{1-p}$ drawn from
different strata's propensities, so they do not generically agree.

The following theorems formalize that a gap exists if and only if
either of these dependencies holds.

\begin{theorem}[Necessary Condition] \label{thm: when gap}
If $P \indep X$ \emph{and} $X \indep Y \mid P$, then $\Delta = 0$.
\end{theorem}

\begin{theorem}[Generic Sufficient Condition]
\label{thm: sufficient for generic}
If either $P \notindep X$ or $X \notindep Y \mid P$, then on a set
of parameters $(\pi, \gamma)$ of full Lebesgue measure satisfying
that condition, $\Delta > 0$.
\end{theorem}

\Cref{fig: when can you mask} illustrates these conditions across
the four canonical structures over $(P, X, Y, D)$: the three under
which masking generically improves on the unaware policy
($X{-}Y$ heterogeneity alone, $X{-}P$ confounding alone, or both),
and the single structure under which it does not (neither).

\begin{figure}[h!]
    \centering
    \scalebox{.64}{\begin{tikzpicture}[
    -latex, auto, 
    node distance = 1.5cm and 1.5cm, 
    on grid, 
    ultra thick,
    state/.style={circle, draw, minimum width=.5cm, ultra thick, fill=white}
]
    \begin{scope}[shift={(0,0)}]
        \node[state] (P-a) {$P$};
        \node[state] (X-a) [below = of P-a] {$X$};
        \node[state] (D-a) [left = of X-a] {$D$};
        \node[state] (Y-a) [right = of X-a] {$Y$};
        \draw[-] (X-a) -- (Y-a);
        \draw[-] (P-a) -- (Y-a);
        \path (P-a) edge (D-a);
        \path (X-a) edge (D-a);
    \end{scope}
    \begin{scope}[shift={(5cm,0)}]
        \node[state] (P-b) {$P$};
        \node[state] (X-b) [below = of P-b] {$X$};
        \node[state] (D-b) [left = of X-b] {$D$};
        \node[state] (Y-b) [right = of X-b] {$Y$};
        \draw[-] (X-b) -- (P-b); 
        \draw[-] (P-b) -- (Y-b);
        \path (P-b) edge (D-b);
        \path (X-b) edge (D-b);
    \end{scope}
    \begin{scope}[shift={(10cm,0)}]
        \node[state] (P-c) {$P$};
        \node[state] (X-c) [below = of P-c] {$X$};
        \node[state] (D-c) [left = of X-c] {$D$};
        \node[state] (Y-c) [right = of X-c] {$Y$};
        \draw[-] (X-c) -- (P-c);
        \draw[-] (X-c) -- (Y-c);
        \draw[-] (P-c) -- (Y-c);
        \path (P-c) edge (D-c);
        \path (X-c) edge (D-c);
    \end{scope}
    \begin{scope}[shift={(15cm,0)}]
        \node[state] (P-d) {$P$};
        \node[state] (X-d) [below = of P-d] {$X$};
        \node[state] (D-d) [left = of X-d] {$D$};
        \node[state] (Y-d) [right = of X-d] {$Y$};
        \draw[-] (P-d) -- (Y-d);
        \path (P-d) edge (D-d);
        \path (X-d) edge (D-d);
    \end{scope}
    \begin{scope}[on background layer]
        \coordinate (top-stretch-abc) at ([yshift=1.2cm]P-a.north);
        \coordinate (top-stretch-d)   at ([yshift=1.2cm]P-d.north);
        \node[
            fit=(P-a)(D-a)(Y-a)(P-b)(D-b)(Y-b)(P-c)(D-c)(Y-c)(top-stretch-abc), 
            rounded corners, 
            fill=red!10, 
            draw=red!70!black,
            line width=2pt,
            inner xsep=10pt,
            inner ysep=5pt 
        ] (red-box-large) {};
        \node[text=red!70!black, anchor=north, yshift=-5pt] at (red-box-large.north) 
            {\Large \textbf{$\mathcal{W}(D_{\text{mask}}) > \mathcal{W}(D_{\text{una}})$ (generically)}};
        \node[
            fit=(P-d)(D-d)(Y-d)(top-stretch-d), 
            rounded corners, 
            fill=blue!10, 
            draw=blue!60!black,
            line width=2pt,
            inner xsep=20pt,
            inner ysep=5pt
        ] (blue-box-small) {};
        \node[text=blue!60!black, anchor=north, yshift=-5pt] at (blue-box-small.north) 
            {\Large \textbf{$\mathcal{W}(D_{\text{mask}}) = \mathcal{W}(D_{\text{una}})$}};
    \end{scope}
\end{tikzpicture}}
    \caption{Four canonical structures over $(P, X, Y, D)$. Masking
    generically improves on the unaware policy whenever at least one
    of $X{-}Y$ heterogeneity or $X{-}P$ confounding is present
    (red box): the leftmost panel has $X{-}Y$ heterogeneity only,
    the second has $X{-}P$ confounding only, and the third has both.
    When neither is present (blue box), masking provides no gain.
    Undirected edges indicate flexibility in causal direction and
    the possibility of unobserved confounding; the directed edges
    $P \rightarrow D$ and $X \rightarrow D$ represent the policy's
    use of $P$ and $X$ as inputs.}
    \label{fig: when can you mask}
\end{figure}

\subsection{The Role of Regulation}

Detecting a nonzero ATE is a $z$-test whose sample complexity does
not depend on $k$, while detecting a masked policy requires a
stratum-wise test that splits the data into $k$ pieces and loses
power as $k$ grows (both tests are in Appendix~\ref{apx: hyp tests}).
Causally masked policies are therefore harder to detect than
unmasked ones. Appendix~\ref{apx: measure theory} formalizes this: optimized
solutions are almost surely stratum-wise unfair regardless of
regulation, and averaged-effect constraints push them into the
masked regime where the same unfairness is statistically harder
to detect.

The geometry of the feasible sets, analyzed in
Appendix~\ref{apx: quantifying feasibility sets}, makes this
concrete: $\varepsilon$-unaware feasible volumes scale as
$O(\varepsilon^k)$ while $\varepsilon$-masking volumes scale as
$O(\varepsilon)$, so the masking constraint admits dramatically
more flexibility under any finite-sample relaxation.

Appendix~\ref{apx: synthetic} performs a synthetic experiment on $100{,}000$ uniformly sampled worlds confirms that small relaxations of
$\varepsilon_{\text{mask}}$ yield substantially larger utility gains
than equivalent relaxations of $\varepsilon_{\text{una}}$, with the
gap widening in $k$.

\section{Real Data Demonstration} \label{sec: empirical}

We demonstrate the persistence of masked policies on the COMPAS
recidivism dataset \citep{angwin_machine_2022} by simulating a parole
policy that minimizes 2-year recidivism. We impose an averaged
constraint with respect to minority status and show that the
resulting policy remains stratum-wise unfair while evading detection for substantially longer than an
unconstrained policy.

\subsection{Setup}

The decision $D \in \{0, 1\}$ indicates parole release. We minimize
$\E[Y \mid d_1]\Pr(d_1)$, where $Y$ is 2-year recidivism, subject
to a minimum release-rate constraint
$\Pr(d_1) = \sum_{x, p}\pi_{x, p}\alpha(x, p) \ge \rho$, where
$\pi_{x, p} = \Pr(x, p)$ is estimated from the data and
$\rho = 0.1$. Strata $X$ are derived from age category, prior
records, charge degree, and sex of the defendant, discretized via
quantile-based binning at varying resolutions $k$.

We evaluate each policy by testing two nested null hypotheses:
\begin{description}[leftmargin=0pt, itemsep=0.5ex]
    \item[$H_{\text{ATE}}$ (averaged):]
    $\sum_{x} \Pr(X = x)(\alpha_{x, 1} - \alpha_{x, 0}) = 0$, the
    averaged regulatory check.
    \item[$H_{\text{CATE}}$ (stratified):]
    $\alpha_{x, 1} - \alpha_{x, 0} = 0$ for all $x$, the strict
    requirement; $H_{\text{CATE}} \Rightarrow H_{\text{ATE}}$.
\end{description}
Rejecting $H_{\text{CATE}}$ while accepting $H_{\text{ATE}}$ is the
signature of a successful mask. We use a $z$-test for $H_{\text{ATE}}$ and a Holm-corrected family of stratum-wise Fisher's exact tests for $H_{\text{CATE}}$
(Appendix~\ref{apx: hyp tests}).

\subsection{Results}

\begin{figure*}[h]
    \centering
    \includegraphics[width=0.95\linewidth]{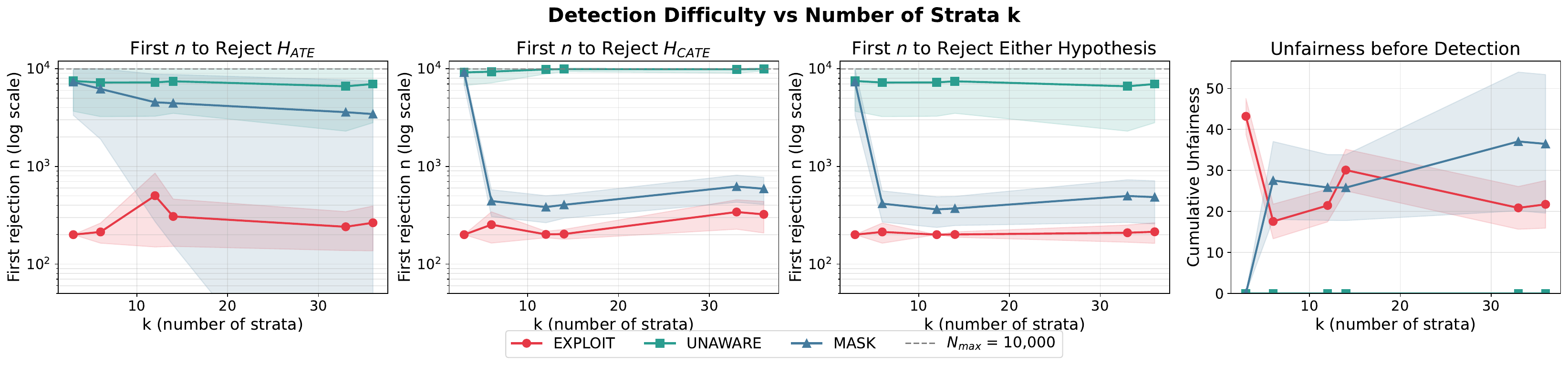}
    \caption{Sample size required to reject each null hypothesis
    for the three strategies (mean $\pm$ standard error). The
    masked strategy rejects $H_{\text{ATE}}$ at sample sizes
    comparable to the unaware strategy, while the exploitative
    strategy is rejected almost immediately. Even under the combined
    test (third panel), the masked strategy persists over twice as long as the exploitative one, and the gap widens
    with $k$. Cumulative unfairness (rightmost panel) is highest
    for the masked strategy because its concealment outlasts its
    detection.}
    \label{fig:Real data}
\end{figure*}

\Cref{fig:Real data} reports the sample size at which each test
first rejects, simulating the longevity of each strategy. The
unaware strategy is rejected last (a chance event under the null),
and the exploitative strategy is rejected first. The masked strategy
rejects $H_{\text{ATE}}$ at sample sizes indistinguishable from the
unaware strategy, despite being strictly unfair within strata.
Under the combined test, the masked strategy persists more than
twice as long as the exploitative one, with the gap widening
for larger $k$ as the stratum-wise test loses power. The rightmost
panel translates this longevity into cumulative unfairness,
$\sum_x |\alpha_{x, 1} - \alpha_{x, 0}|$ accumulated until rejection:
the masked policy contributes the most total unfairness, because
its concealment outlasts its detection.

\section{Discussion}\label{sec:discussion}

This paper introduces the causal masking problem, formulates it as
an LP, and examines its properties. We showed that averaged-effect
regulation almost always drives optimizers toward stratum-wise unfair
solutions, and that detecting such unfairness from observed decisions
is information-theoretically difficult. As a result, masked unfair
policies can remain undetected for long periods, with the detection
gap growing in $k$. More complex observed covariates widen the gap
between unmasked and masked unfair policies while also making masking
more profitable. If we restrict ourselves to decision-level
statistics, causal masking is likely and difficult to detect or regulate.

These limitations are grounded in the statistical difficulty of
conditional independence testing and therefore extend beyond causal
fairness to any setting where regulation is enforced through
conditional parity on aggregated outputs. The implications cover
fairness applications such as college admissions and criminal
justice (as discussed in this paper), affirmative action
\citep{holzer_assessing_2000}, and the gender wage gap
\citep{pendakur_glass_2010} --- all of which are predominantly
analyzed through conditional (adjusted) averages. Beyond fairness,
the same structure constrains regulation of insider trading,
classified-information leakage, and any other domain where dependence
on a prohibited variable is permitted only through an allowable
mediator. In all of these settings, our results argue that effective
regulation must operate at the model level, where the decision
mechanism itself is accessible, rather than only at the level of its
outputs.

The masking advantage is not recovered by simpler
randomization-based defenses from the information leakage literature
\citep{alvim_information_2022}. Mixing between exploit and unaware
policies traces out a Pareto frontier in the utility-detectability
plane; masking sits strictly above it, achieving higher utility than
any mixing rate at the same detection longevity
(Appendix~\ref{apx:mixed_comparison}). The masking critique is a
structural property of how averaged constraints interact with
reward-driven optimization, not something randomization can address.

\subsection{Limitations}

Our results are derived in a deliberately simplified setting: a
discrete stratum variable $X$, a binary prohibited variable $P$, a
binary decision $D$, randomized policies, and a fixed participation
rate. Several aspects of this simplification merit explicit comment.

\paragraph{Discreteness and randomization.}
A natural objection is that masking depends on stratum-level
randomization, which a regulator could rule out by requiring
identical cases to be treated identically. This defense fails on
practical grounds: in any realistic setting, no two cases are
identical along all observable features and
an optimizer with high-dimensional covariates can simulate
randomization deterministically by hashing low-importance features.
The discreteness and randomization assumptions in our model are
modeling conveniences, not real constraints on the masking critique.

\textbf{Continuous and nonlinear extensions.} Our theory and our
hypothesis tests assume discrete strata. Real-world settings often
involve continuous covariates and nonlinear outcome relationships.
On balance, these complications make the regulator's problem harder
rather than easier: continuous covariates expand the dimensions
along which masked policies can offset CATEs, and nonlinearity in
the outcome amplifies the heterogeneity that masking exploits. The
heterogeneity tests cited in \Cref{subsec: related works} provide a starting point for extending our diagnostic framework to continuous settings.

\paragraph{Hidden confounding does not exacerbate masking.}
The masking advantage is driven by \emph{observed} dependence
between $P$ and $X$ --- it is precisely because the optimizer can
act differentially within observed strata that masking is profitable.
Unobserved confounders contribute variation that neither side can
condition on and so cannot be recruited into a masking strategy. Our
experiments on $k$ confirm this: a coarser $X$ pushes fine-grained
substructure into the unobserved residual, and the masking gap
shrinks accordingly.

\subsection{Future Directions}

Several extensions follow naturally from this work.

\paragraph{Success-rate maximization.}
An optimizer without a fixed participation rate must instead
maximize a ratio (graduation rate, win rate, expected GPA), turning
the LP into a linear-fractional program with different geometry.
The same setting also covers strategic problems like card-counting,
where the optimizer's choice of \emph{when} to act is itself a
decision variable.

\paragraph{Beyond averaged constraints.}
Our masking constraint targets the ATE, but other measures of
global dependence --- odds ratios detected via the
Cochran-Mantel-Haenszel test \citep{zhang_generalized_1997},
conditional mutual information, or $f$-divergence-based measures
--- impose distinct constraint geometries. Extending the masking
framework to these quantities would clarify which forms of summary
regulation are robust to optimizer adaptation and which are merely
substitutes for the ATE-based regulation we study here.

\bibliography{biblio}


\appendix

\section{Deferred Proofs}\label{apx: deferred proofs}
\subsection{Proof of Theorem~\ref{thm: fair vs mask corrected}} \label{apx: prove perf gap}
\begin{proof}
The proof proceeds by comparing the utility of the optimal fair policy to a specially constructed, valid masking policy candidate, $D''_{p,j}$.

\textbf{1. Utility of the Optimal Fair Policy:}
Under the small $\rho$ assumption, the optimal fair policy is to participate only in stratum $i$ with rates $\alpha_{i,1} = \alpha_{i,0} = \rho / \Pr(X=i)$. The utility is:
\[ \mathcal{W}(D_{\text{fair}}) = \rho \cdot w_i^{\text{avg}}. \]

\textbf{2. Construct a Valid Masking Candidate Policy, $D''_{p,j}$:}
We build a candidate policy from scratch that is inspired by an arbitrage between states $(i,1)$ and $(j,0)$ (the argument for $(i,0)$ and $(j,1)$ is symmetric).

\textit{Step A: Define the Unscaled Arbitrage Policy ($D'$).}
Let's construct an intermediate policy $D'$ that has non-zero participation only for rates $\alpha'_{i,1}$ and $\alpha'_{j,0}$. To satisfy the zero ATE constraint, these rates must obey the balance equation: $\Pr(X=i)\alpha'_{i,1} = \Pr(X=j)\alpha'_{j,0}$. To make this policy as extreme as possible while keeping all rates $\le 1$, we set one rate to 1 and solve for the other. This results in the policy:
\[ \alpha'_{i,1} = \frac{\min(\Pr(X=i), \Pr(X=j))}{\Pr(X=i)} \quad \text{and} \quad \alpha'_{j,0} = \frac{\min(\Pr(X=i), \Pr(X=j))}{\Pr(X=j)}. \]
For example, if $\Pr(X=i)$ is larger, then $\alpha'_{i,1} < 1$ and $\alpha'_{j,0} = 1$.

\textit{Step B: Calculate the Rate of Return of $D'$.}
The utility of this unscaled policy is:
\begin{align*}
    \mathcal{W}(D') &= \alpha'_{i,1}\gamma_{i,1}\pi_{i,1} + \alpha'_{j,0}\gamma_{j,0}\pi_{j,0} \\
    &= \frac{\min(\dots)}{\Pr(x_i)}\gamma_{i,1}\Pr(P=1|x_i)\Pr(x_i) + \frac{\min(\dots)}{\Pr(x_j)}\gamma_{j,0}\Pr(P=0|x_j)\Pr(x_j) \\
    &= \min(\Pr(X=i), \Pr(X=j)) \cdot [w_{i,1} + w_{j,0}].
\end{align*}
The participation rate of this policy is:
\begin{align*}
    \rho' &= \alpha'_{i,1}\pi_{i,1} + \alpha'_{j,0}\pi_{j,0} \\
    &= \frac{\min(\dots)}{\Pr(X=i)}\Pr(P=1|x_i)\Pr(x_i) + \frac{\min(\dots)}{\Pr(X=j)}\Pr(P=0|x_j)\Pr(x_j) \\
    &= \min(\Pr(X=i), \Pr(x_j)) \cdot [\Pr(P=1|i) + \Pr(P=0|j)].
\end{align*}
The key insight is that the rate of return, $\mathcal{W}(D')/\rho'$, is symmetric and simplifies because the ``min'' term cancels out:
\[ \frac{\mathcal{W}(D')}{\rho'} = \frac{w_{i,1} + w_{j,0}}{\Pr(P=1|i) + \Pr(P=0|j)} = R_0(j). \]

\textit{Step C: Scale to Create Valid Candidate $D''$.}
We now define our final candidate policy $D''_{0,j}$ by scaling all rates of $D'$ by a factor of $(\rho / \rho')$. This policy is guaranteed to be valid if we assume $\rho$ is sufficiently small (specifically, $\rho \le \rho'$). By construction, its participation rate is exactly $\rho$, and its utility is:
\[ \mathcal{W}(D''_{0,j}) = \rho \cdot R_0(j). \]

\textbf{3. Bounding the Gap:}
The optimal masking policy's utility, $\mathcal{W}(D_{\text{mask}})$, must be at least as high as the utility of the best of these constructed candidates. The performance gap $\Delta$ is therefore bounded by:
\begin{align*}
\Delta &= \mathcal{W}(D_{\text{mask}}) - \mathcal{W}(D_{\text{fair}}) \\
&\ge \max_{j \neq i, p \in \{0,1\}} \left\{ \mathcal{W}(D''_{p,j}) \right\} - \mathcal{W}(D_{\text{fair}}) \\
&= \rho \cdot \max_{j \neq i, p \in \{0,1\}} \left\{ R_p(j) \right\} - \rho \cdot w_i^{\text{avg}}.
\end{align*}
Factoring out $\rho$ and ensuring the gap is non-negative completes the proof.
\end{proof}

\subsection{Proof of Theorem~\ref{thm: when gap}}\label{apx: when gap proof}
\begin{proof}
The two assumptions translate directly to our model parameters:
\begin{enumerate}
    \item $P \indep X \implies \Pr(P=p|X=x) = \Pr(P=p)$ for all strata $x$.
    \item $X \indep Y \mid P \implies \mathbb{E}[Y|X=x, P=p] = \mathbb{E}[Y|P=p]$. In our notation, this means $\gamma_{x,p}$ does not depend on $x$, so we write it as $\gamma_p$.
\end{enumerate}
Under these assumptions, the context-weighted reward is $w_{x,p} = \gamma_p \Pr(P=p)$, which is constant for all strata $x$. Consequently, the average reward $w_x^{\text{avg}} = w_{x,0} + w_{x,1} = \gamma_0\Pr(P=0) + \gamma_1\Pr(P=1)$ is also constant for all $x$. The fair policy is indifferent to which stratum it chooses. Let's pick $i$.

Now, let's analyze the arbitrage rate of return from \Cref{thm: fair vs mask corrected}:
\begin{align*}
    R_p(j) &= \frac{w_{i,\overline{p}} + w_{j,p}}{\Pr(P=\overline{p}|x_i) + \Pr(P=p|x_j)} \\
    &= \frac{\gamma_{\overline{p}}\Pr(P=\overline{p}) + \gamma_p\Pr(P=p)}{\Pr(P=\overline{p}) + \Pr(P=p)} \\
    &= \gamma_{\overline{p}}\Pr(P=\overline{p}) + \gamma_p\Pr(P=p) = w_i^{\text{avg}}.
\end{align*}
Since $R_p(j) = w_i^{\text{avg}}$ for all possible swaps, the lower bound on the gap from \Cref{thm: fair vs mask corrected} is $\rho \cdot \max(0, 0) = 0$.

To show the gap is exactly zero, we note that any valid masking policy must satisfy the zero ATE constraint. Under our assumptions, the decisions $\alpha_{x,p}$ for a given $p$ will be the same across all $x$ (since $w_{x,p}$ is constant). Let this policy be $(\alpha_0, \alpha_1)$. The ATE constraint becomes $(\alpha_1 - \alpha_0) \sum_x \Pr(X=x) = \alpha_1 - \alpha_0 = 0$. This forces $\alpha_1 = \alpha_0$, meaning any optimal masking policy must also be a fair policy. Thus, $\mathcal{W}(D_{\text{mask}}) = \mathcal{W}(D_{\text{fair}})$ and $\Delta=0$.
\end{proof}

\subsection{Proof of Theorem \ref{thm: sufficient for generic}} \label{apx: generic gap}
\begin{proof}
The proof relies on showing that if either dependency holds, the "no-gap" condition ($\Delta=0$) imposes a non-trivial equality constraint on the problem parameters, which defines a set of Lebesgue measure zero. We use the bound from \Cref{thm: fair vs mask corrected}, which guarantees a positive gap if $R_p(j) > w_i^{\text{avg}}$ for some swap $(j,p)$, where $i=\arg\max_x w_x^{\text{avg}}$.

\textbf{Case 1: Confounding exists ($P \notindep X$).}
Assume for simplicity $X \indep Y \mid P$ (so $\gamma_{x,p} = \gamma_p$). The condition $P \notindep X$ means that for some $j \neq i$, $\Pr(P|j) \neq \Pr(P|i)$. The arbitrage return $R_p(j)$ is a weighted average of $\gamma_p$ and $\gamma_{\overline{p}}$, while $w_i^{\text{avg}}$ is a different weighted average of the same values. These two averages will be equal only if $\gamma_p = \gamma_{\overline{p}}$ or if the propensity scores perfectly cancel out the differences in the objective function. For generic, continuous parameters $(\pi, \gamma)$, the strict inequality $R_p(j) \neq w_i^{\text{avg}}$ will hold. There will almost surely be a swap for which the inequality is favorable ($R_p(j) > w_i^{\text{avg}}$), creating a positive gap. The condition $R_p(j) = w_i^{\text{avg}}$ defines a measure-zero set.

\textbf{Case 2: Heterogeneous effects exist ($X \notindep Y \mid P$).}
Assume for simplicity $P \indep X$ (so $\Pr(P|x) = \Pr(P)$). The condition $X \notindep Y \mid P$ means that for some $p$, $\gamma_{x,p}$ is not constant across strata. The arbitrage return $R_p(j)$ involves terms like $\gamma_{j,p}$ and $\gamma_{i,\overline{p}}$, while the fair return $w_i^{\text{avg}}$ involves $\gamma_{i,p}$ and $\gamma_{i,\overline{p}}$. For a generic choice of reward parameters, it would be a measure-zero coincidence for the specific combination of terms in $R_p(j)$ to exactly equal the combination in $w_i^{\text{avg}}$. Therefore, a profitable swap will almost surely exist, guaranteeing $\Delta > 0$.

Since a strictly positive gap is generated (generically) if either condition holds, a zero-gap outcome represents a "perfect alignment" of parameters that is broken by any small, random perturbation.
\end{proof}

\section{Genericity of Constraints and Solutions} \label{apx: measure theory}

\subsection{Genericity of Constraints}

A policy is a vector $\vec{\alpha} = (\alpha_{x_1,0}, \alpha_{x_1,1}, \dots, \alpha_{x_k,0}, \alpha_{x_k,1})$ in the policy space $\mathcal{P} = [0, 1]^{2k}$. We first show that an arbitrarily chosen policy is almost never fair or masked.

\begin{theorem}
The set of policies $\vec{\alpha} \in [0, 1]^{2k}$ that satisfy (a) the fairness constraint or (b) the masking constraint are sets of Lebesgue measure zero in $\mathcal{P}$.
\end{theorem}

\begin{proof}
The space of all possible policies $\mathcal{P} = [0, 1]^{2k}$ is a $2k$-dimensional hypercube with volume 1.

\textbf{(a) Fairness Constraint:} The fairness constraint is $A \indep U \mid X$, which corresponds to the set of $k$ linear equality constraints:
\[ \alpha_{x,1} - \alpha_{x,0} = 0 \quad \text{for each } x \in \{x_1, \dots, x_k\}. \]
These $k$ constraints are independent, as each involves a disjoint pair of coordinates of $\vec{\alpha}$. Geometrically, the set of policies satisfying these constraints is the intersection of $k$ distinct hyperplanes with the hypercube $\mathcal{P}$. This intersection forms a $k$-dimensional manifold within the $2k$-dimensional space $\mathcal{P}$. For any $k \ge 1$, we have $k < 2k$. A $k$-dimensional manifold has zero Lebesgue measure in a $2k$-dimensional space. Thus, the set of fair policies has measure zero.

\textbf{(b) Masking Constraint:} The masking constraint corresponds to a single linear equality:
\[ \sum_{x=1}^k \Pr(X=x) (\alpha_{x,1} - \alpha_{x,0}) = 0, \]
where $\Pr(X=x) = \pi_{x,0} + \pi_{x,1}$. Assuming not all $\Pr(X=x)$ are zero, this is a single non-trivial linear constraint on the coordinates of $\vec{\alpha}$. Geometrically, this equation defines a $(2k-1)$-dimensional hyperplane in $\mathbb{R}^{2k}$. The set of policies satisfying the constraint is the intersection of this hyperplane with the hypercube $\mathcal{P}$, which is a $(2k-1)$-dimensional manifold. Since $2k-1 < 2k$, this set has Lebesgue measure zero in $\mathcal{P}$.
\end{proof}

\subsection{Genericity of Optimal Solutions}

We now consider the "world parameters" $(\pi, \gamma)$ which define the problem instance. Let the space of all such parameters be $\Theta$, a subset of a high-dimensional Euclidean space with a well-defined Lebesgue measure.

\begin{theorem}
The set of world parameters $(\pi, \gamma) \in \Theta$ for which the optimal exploit policy also satisfies (a) the fairness constraint or (b) the masking constraint is a set of Lebesgue measure zero.
\end{theorem}

\begin{proof}
Let $w'_{x,u} \coloneqq \gamma_{x,u}\pi_{x,u}$. The objective is $\mathcal{W}(D) = \sum_{x,u} w'_{x,u} \alpha_{x,u}$. For a generic choice of parameters, the values of $w'_{x,u}$ will be unique. In this case, the optimal exploit policy is a pure strategy: set $\alpha_{x^*, u^*} = 1$ for $(x^*, u^*) = \arg\max_{x,u} w'_{x,u}$, and all other $\alpha_{x,u} = 0$.

\textbf{(a) Fairness Constraint:} The pure exploit policy sets one coordinate to 1 and the rest to 0. For this to satisfy the fairness constraints $\alpha_{x,1} = \alpha_{x,0}$ for all $x$, every coordinate must be 0 (if all $w'_{x,u} \le 0$) or 1. A non-trivial pure exploit policy necessarily violates the constraint at stratum $x^*$, where it sets $(\alpha_{x^*,1}, \alpha_{x^*,0})$ to either $(1,0)$ or $(0,1)$.

A non-pure exploit policy is optimal only if there is a tie for the maximum value of $w'_{x,u}$. For an optimal exploit solution to be fair, it must be that a fair policy is also an optimal exploit solution. This requires a specific alignment of parameters. For instance, if the optimal fair policy is $\alpha_{x',1}=\alpha_{x',0}=1$ for some $x'$, then for this to also be an optimal exploit solution, we need $w'_{x',1} + w'_{x',0} \ge w'_{x,u}$ for all $(x,u)$ and for any other fair policy. A necessary condition for such a policy to be expressible as an optimal exploit solution is a tie, such as $w'_{x',1} = w'_{x',0}$ being the two highest values. This equality, $\gamma_{x',1}\pi_{x',1} = \gamma_{x',0}\pi_{x',0}$, is a single non-trivial equality constraint on the continuous parameters $(\pi, \gamma)$ and thus defines a set of Lebesgue measure zero in $\Theta$.

\textbf{(b) Masking Constraint:} Plugging the pure exploit policy $\alpha_{x^*,u^*}=1$ into the masking constraint gives $\Pr(X=x^*)(\delta_{u^*,1} - \delta_{u^*,0})=0$, which implies $\Pr(X=x^*)=0$, a trivial case. Thus, a pure exploit policy cannot satisfy the masking constraint.

The constraint can only be satisfied if the optimal exploit policy is non-pure, which requires a tie, e.g., $w'_{x_A, u_A} = w'_{x_B, u_B} = \max_{x,u} w'_{x,u}$. This is an equality constraint on $(\pi, \gamma)$, e.g., $\gamma_{x_A, u_A}\pi_{x_A, u_A} = \gamma_{x_B, u_B}\pi_{x_B, u_B}$, which defines a measure-zero set.
\end{proof}

\begin{theorem}
The set of world parameters $(\pi, \gamma) \in \Theta$ for which the optimal masking policy is also a fair policy is a set of Lebesgue measure zero.
\end{theorem}

\begin{proof}
The fair policies are a strict subset of the feasible set for the masking problem. Therefore, $\mathcal{W}(A_{\text{mask}}) \ge \mathcal{W}(A_{\text{fair}})$. The solutions are identical if and only if the maximum over the larger set happens to lie in the smaller subset.

The optimal masking policy is a double-threshold strategy: $\alpha_{x,1} = \I(w_{x,1} > \lambda)$ and $\alpha_{x,0} = \I(w_{x,0} > -\lambda)$, where $w_{x,u} \coloneqq \gamma_{x,u}\Pr(u|x)$ and $\lambda$ is a Lagrange multiplier chosen to satisfy the constraint (we ignore measure-zero cases where the decision is fractional).

For this policy to be fair, it must be that $\alpha_{x,1} = \alpha_{x,0}$ for all $x$. This means that for the optimal $\lambda^*$, every stratum $x$ must have a policy of either $(1,1)$ or $(0,0)$. No stratum can have an "unfair" policy of $(1,0)$ or $(0,1)$. A policy is $(1,0)$ if $w_{x,1} > \lambda^*$ and $w_{x,0} \le -\lambda^*$. A policy is $(0,1)$ if $w_{x,1} \le \lambda^*$ and $w_{x,0} > -\lambda^*$.

The masking LP is free to select such "unfair" policies for individual strata if doing so increases the global objective, provided the choices can be balanced to satisfy the zero ATE constraint. For instance, suppose stratum $x_A$ has a large positive $w_{x_A,1}$ and stratum $x_B$ has a large positive $w_{x_B,0}$. The masking LP can potentially set $\alpha_{x_A,1}=1$ and $\alpha_{x_B,0}=1$ (and others zero), yielding a high objective value. This policy is unfair. It is feasible for the masking problem if $\Pr(X=x_A)(1-0) + \Pr(X=x_B)(0-1) = 0$, i.e., if $\Pr(X=x_A) = \Pr(X=x_B)$.

The optimal masking solution is fair only if, for the given parameters $(\pi, \gamma)$, no such profitable "unfair" combination of choices exists. The condition where an unfair policy (like the one above) gives an objective value *equal to* the optimal fair policy's objective value is an equality constraint on the parameters $(\pi, \gamma)$. For instance:
\[ \gamma_{x_A,1}\pi_{x_A,1} + \gamma_{x_B,0}\pi_{x_B,0} = \max_{x} \left( \gamma_{x,1}\pi_{x,1} + \gamma_{x,0}\pi_{x,0} \right) \]
Such equations define the boundaries between regions in the parameter space $\Theta$ where different types of policies are optimal. These boundaries are lower-dimensional manifolds and thus have Lebesgue measure zero. For a generic choice of parameters, one side of the inequality will hold, and it is overwhelmingly likely that there exists some profitable, unfair trade-off. The masking solution will thus almost never be a fair one.
\end{proof}

\section{The Geometry of Fair and Masked Policies} \label{apx: quantifying feasibility sets}
Let the world of interest be defined by a policy vector $\vec{\alpha} \in [0,1]^{2k}$, with constraints carving out a feasible region within this hypercube. The size of this region represents the ``freedom'' the optimizer has to find a good solution. When constraints are perfectly enforced ($\varepsilon=0$), the feasible sets become lower-dimensional affine subspaces. Their dimension quantifies their size.

\begin{theorem}[Dimensionality of Feasible Sets] \label{thm: dim of feasible sets}
Let $V_{\text{fair}}$ and $V_{\text{mask}}$ be the feasible sets for the perfectly fair and masking problems ($\varepsilon_{\text{fair}} = \varepsilon_{\text{mask}}=0$), respectively.
\begin{enumerate}
    \item \textbf{Containment:} $V_{\text{fair}} \subseteq V_{\text{mask}}$.
    \item \textbf{Dimensionality:} The dimension of $V_{\text{fair}}$ is $k-1$, while the dimension of $V_{\text{mask}}$ is $2k-2$.
    \item \textbf{Growth of Freedom:} The masking problem has $k-1$ more dimensions of freedom than the fair problem.
\end{enumerate}
\end{theorem}
\begin{proof}
(1) Any policy satisfying the $k$ local fairness constraints, $\alpha_{x,1}-\alpha_{x,0}=0$, also satisfies the global masking constraint, $\sum_x \Pr(X=x)(\alpha_{x,1}-\alpha_{x,0})=0$.
(2) The policy space is $2k$-dimensional. The fair problem imposes $k$ local fairness constraints and 1 global participation constraint, for a total of $k+1$ independent linear equality constraints. The resulting dimension is $2k - (k+1) = k-1$. The masking problem imposes 1 global ATE constraint and 1 global participation constraint, for 2 total constraints. The dimension is $2k-2$.
(3) The difference is $(2k-2) - (k-1) = k-1$.
\end{proof}

This result shows that the masking policy's advantage, in terms of the dimensionality of its search space, grows linearly with the number of strata, $k$. When $\varepsilon > 0$, the constraints define ``slabs'' instead of thin hyperplanes, and the feasible sets become full-dimensional bodies. We can compare their volumes.

\begin{proposition}[Volume of Feasible Sets]
Let $\text{Vol}(V(\varepsilon))$ be the volume of the feasible set for a given $\varepsilon>0$. For a small $\varepsilon$, the volumes scale as:
\begin{enumerate}
    \item $\text{Vol}(V_{\text{fair}}(\varepsilon)) \propto \varepsilon^k$
    \item $\text{Vol}(V_{\text{mask}}(\varepsilon)) \propto \varepsilon^1$
\end{enumerate}
\end{proposition}
\begin{proof}[Proof (Sketch)]
The $\varepsilon$-fair problem constrains the policy vector $\vec{\alpha}$ to lie within $k$ independent slabs, each of width proportional to $\varepsilon$. The volume of their intersection thus scales with $\varepsilon^k$. The $\varepsilon$-masking problem imposes only a single slab constraint, so its volume scales linearly with $\varepsilon$.
\end{proof}
The ratio of the feasible volumes, $\text{Vol}(V_{\text{mask}}) / \text{Vol}(V_{\text{fair}}) \propto 1/\varepsilon^{k-1}$, is therefore immense. This disparity in ``freedom to operate'' provides a clear geometric reason for the masking policy's ability to find significantly better solutions.

\section{Synthetic Experiments} \label{apx: synthetic}

This appendix provides the full synthetic-data experiment summarized
in Section~\ref{sec: theory}. We sample $100{,}000$ worlds uniformly
with $\vec\pi$ on the simplex, $\vec\gamma \in [0, 1]^{2k}$, and
$\rho = 0.25$, and compute $D_{\text{exploit}}$,
$D_{\text{una}}(\varepsilon_{\text{una}})$, and
$D_{\text{mask}}(\varepsilon_{\text{mask}})$ as the LP solutions of
Propositions~\ref{prop: opt exploit}, \ref{prop: opt unaware},
and~\ref{prop:mask}. We normalize each policy's utility so that
$\mathcal{W}(D_{\text{una}}(0)) = 0$ and
$\mathcal{W}(D_{\text{exploit}}) = 1$, and repeat the experiment for
$k = 2$ (top row of \Cref{fig: six heat plots}) and $k = 10$ (bottom
row).

\begin{figure*}[h]
    \centering
    \includegraphics[width=\linewidth]{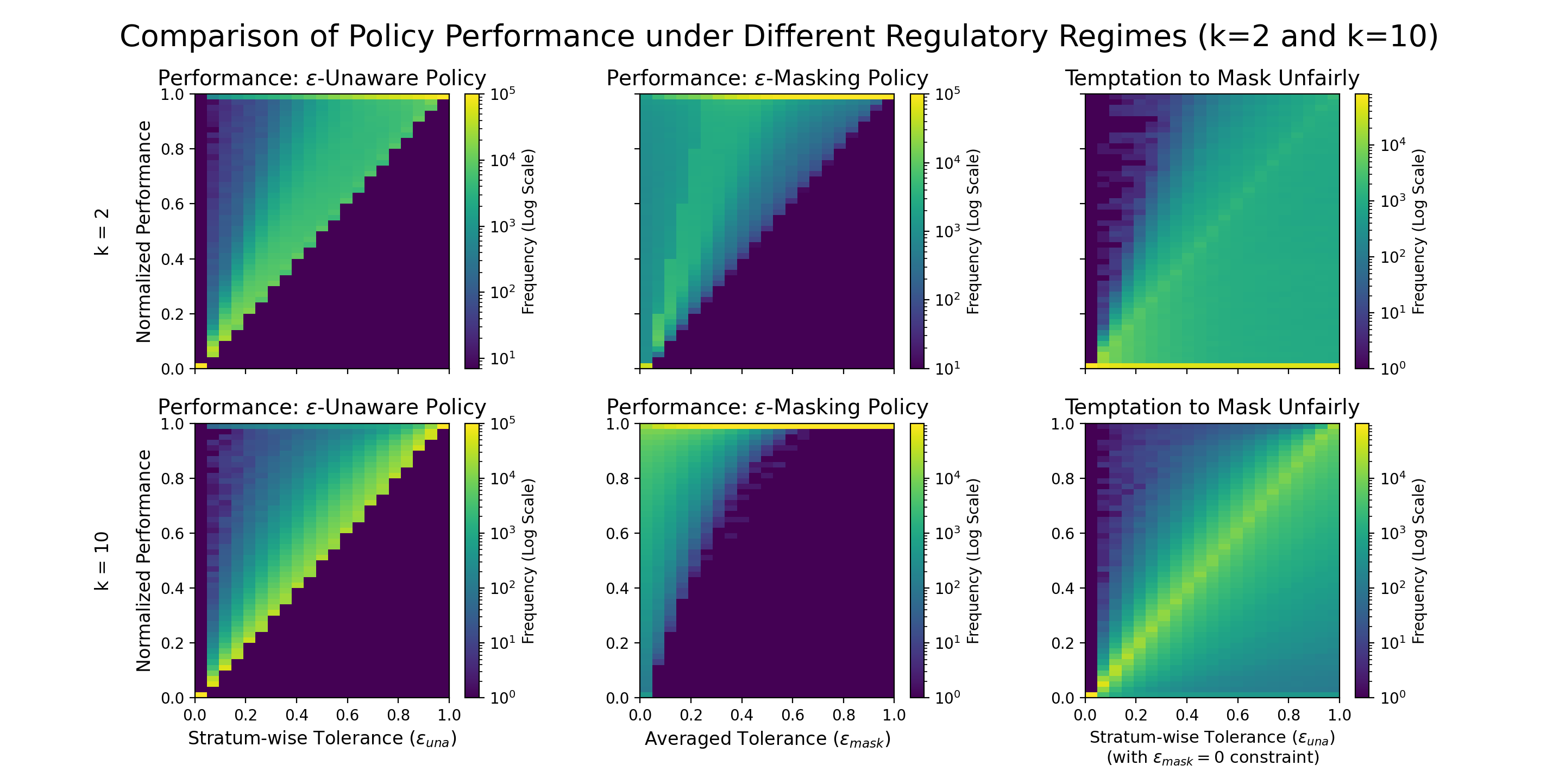}
    \caption{Synthetic experiments comparing the utility achieved
    under three constraint configurations, for $k = 2$ (top) and
    $k = 10$ (bottom). The first column varies
    $\varepsilon_{\text{una}}$ alone; the second varies
    $\varepsilon_{\text{mask}}$ alone; the third fixes
    $\varepsilon_{\text{mask}} = 0$ and varies $\varepsilon_{\text{una}}$.}
    \label{fig: six heat plots}
\end{figure*}

The first column shows utility as $\varepsilon_{\text{una}}$ varies
with no averaged constraint imposed: relaxing the stratum-wise
constraint yields a moderate, gradual improvement. The second column
shows utility as $\varepsilon_{\text{mask}}$ varies with no
stratum-wise constraint: relaxing the averaged constraint yields
much larger gains for the same nominal tolerance, and the disparity
widens with $k$. The third column fixes $\varepsilon_{\text{mask}} = 0$
and varies $\varepsilon_{\text{una}}$, isolating the question of how
much unfairness an optimizer can recover \emph{within} strictly
masked policies; the gain is approximately linear in
$\varepsilon_{\text{una}}$. Together, these results show that an
averaged-only regulation leaves substantial room for utility
improvements that would be statistically undetectable under the same
regulatory test.

\section{Hypothesis Tests}\label{apx: hyp tests}

\subsection{Testing for Zero ATE}
In the setting where $P$ is binary, we employ a Z-test to test the
null hypothesis $H_{\text{ATE}}: \text{ATE} = 0$.
The test statistic is
\[
\widehat{\text{ATE}} = \sum_{i = 1}^k \hat{p}_{x_i} \, \hat{\tau}_{x_i}
\]
where
\[
\hat{p}_x = \frac{1}{n}\sum_{i=1}^n \mathbf{1}_{\{X_i = x\}}, \qquad
x \in \{0, 1\},
\]
and
\[
\hat{\tau}_x = \frac{\sum_{i: X_i = x,\; P_i = 1} D_i}{\sum_{i: X_i = x,\; P_i = 1} 1}
             - \frac{\sum_{i: X_i = x,\; P_i = 0} D_i}{\sum_{i: X_i = x,\; P_i = 0} 1}.
\]
The Z-statistic is then $Z = \widehat{\text{ATE}} \big/
\widehat{\text{SE}}(\widehat{\text{ATE}})$, which is asymptotically
standard normal under the null hypothesis $H_{\text{ATE}}$. We use
a significance level of $5\%$.

\subsection{Testing for Conditional Independence}
A more nuanced way to detect unfairness is to directly test the
null hypothesis of conditional independence,
$H_{\text{CATE}}: D \perp P \mid X$. For each stratum $x \in \nu$,
where $\nu$ denotes the set of strata with observations in both
the $P = 1$ and $P = 0$ groups, we apply Fisher's exact test
\citep{mcdonald2014handbook} for independence between $D$ and $P$
within the stratum, obtaining a $p$-value $p_x$.

Aggregating these per-stratum tests into a joint test for
$H_{\text{CATE}}$ is a multiple-hypothesis testing problem. A
combined-$p$-value approach such as Fisher's method
($S = -2 \sum_{x \in \nu} \log p_x \sim \chi^2_{2|\nu|}$ under the
null) is inappropriate here: the masking constraint mechanically
induces negative dependence between per-stratum CATEs --- a
positive CATE in one stratum must be offset by a negative CATE
elsewhere to satisfy the averaged constraint --- so the
independence assumption underlying Fisher's combination fails on
precisely the alternatives we wish to detect.

We instead control the family-wise error rate using Holm's
step-down procedure \citep{holm_simple_1979}, which is valid under
arbitrary dependence among the per-stratum $p$-values. Sorting
them as $p_{(1)} \le p_{(2)} \le \cdots \le p_{(|\nu|)}$, we reject
$H_{\text{CATE}}$ at level $\alpha = 0.05$ whenever there exists
$i$ such that $p_{(i)} \le \alpha / (|\nu| - i + 1)$.

\section{Comparison with Mixed Strategies from Information Leakage Games}
\label{apx:mixed_comparison}

Information leakage games \citep{alvim_information_2022} formalize a
defender who randomizes among policies to limit the information an
adversary can recover. A natural question, given the structural
parallel between leakage games and our setting, is whether such a
randomization strategy --- mixing between an exploit policy and an
unaware policy on a per-decision basis --- could achieve the same
effect as causal masking, without the explicit construction of
offsetting CATEs across strata.

We test this on synthetic data with $K = 3$ strata and fixed
participation rate $\rho = 0.3$, drawing $100$ random worlds with
$\pi$ sampled from a uniform Dirichlet and $\gamma \in [0, 1]^{2K}$.
For each world we compute the optimal exploit, unaware, and mask
policies as defined in \Cref{sec: LP formulation}, plus a
\emph{mixed} strategy that, for each decision, samples from the
exploit policy with probability $\nicefrac{1}{3}$ and from the
unaware policy otherwise. We then simulate decisions across a
growing sample size $N$ and track three quantities: the
Holm-adjusted $p$-value for $H_{\text{CATE}}$, a stratified-effect
$p$-value (averaging CATEs uniformly across strata), and the
normalized cumulative utility.

\begin{figure}[h]
    \centering
    \includegraphics[width=\linewidth]{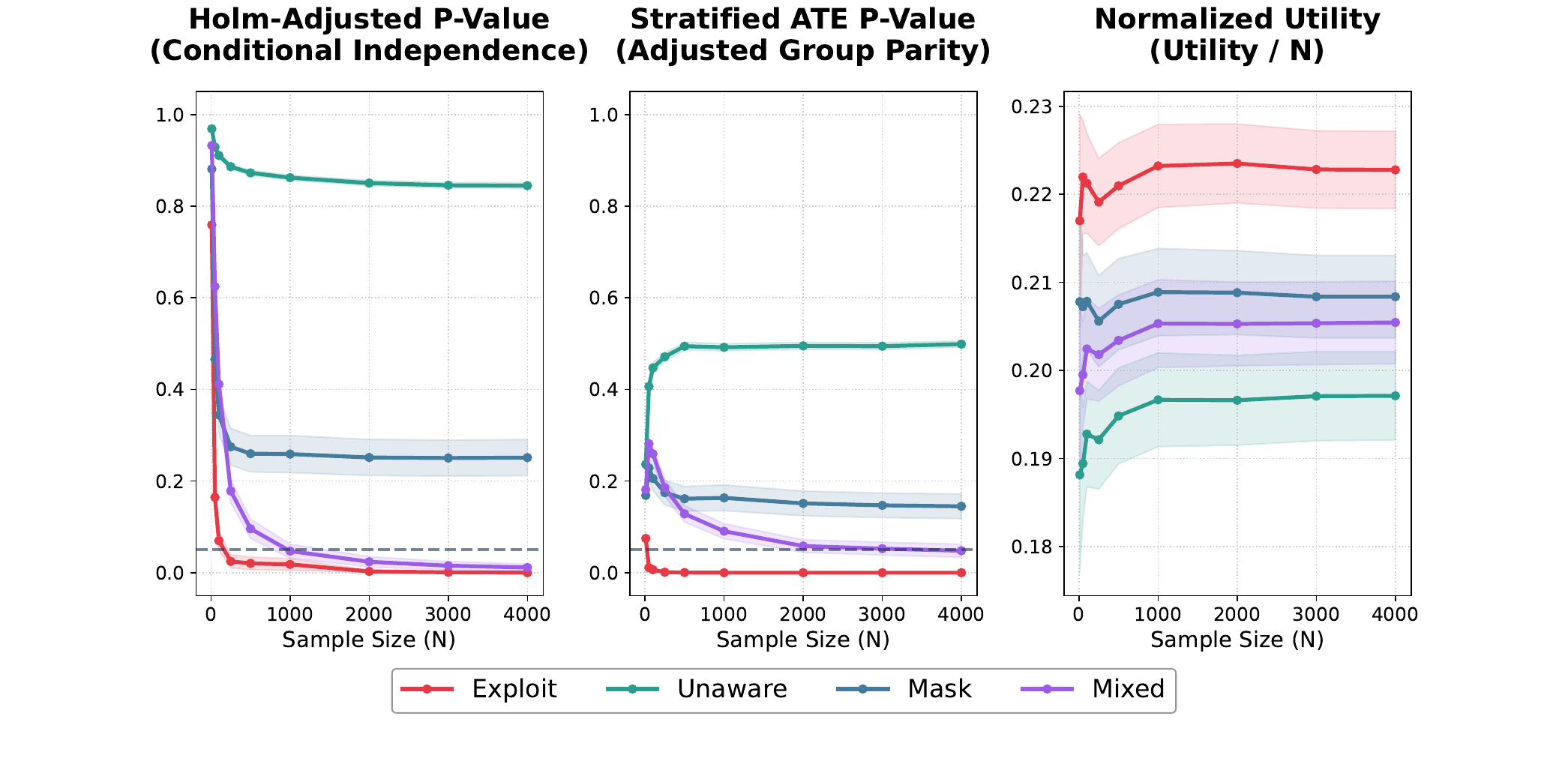}
    \caption{Comparison of exploit, unaware, mask, and mixed strategies
    as the sample size $N$ grows, on synthetic data with $K = 3$ and
    $\rho = 0.3$. Left: Holm-adjusted $p$-value for $H_{\text{CATE}}$
    (conditional independence). Middle: stratified-effect $p$-value
    using uniformly-weighted CATEs. Right: normalized cumulative
    utility per decision. The mask strategy retains the high
    conditional-independence $p$-value of the unaware policy while
    achieving utility comparable to the mixed strategy, demonstrating
    that masking dominates randomization on the utility-detectability
    frontier.}
    \label{fig:mixed_comparison}
\end{figure}

The mixed strategy traces out a Pareto frontier between the exploit
and unaware extremes: by tuning the mixing rate, the optimizer can
move smoothly along the line connecting them in the
utility-detectability plane. The masked policy sits \emph{above}
this frontier. At any given detection longevity --- the sample size
at which $H_{\text{CATE}}$ is first rejected --- the masked policy
achieves higher utility than any mixed strategy of unconstrained
policies, and at any given utility level, the masked policy survives
substantially longer before being detected.

The reason is structural. Mixing between exploit and unaware produces
CATEs that, in expectation, are simply rescaled versions of the
exploit policy's CATEs: the same heterogeneity, attenuated. The
conditional-independence test detects these CATEs at the same rate
as the exploit policy's, just with a smaller effective sample size.
Masking, by contrast, constructs CATEs whose signs are arranged to
cancel under the regulator's averaging weights, so the detection
slowdown is geometric rather than a constant-factor amplitude
reduction. No mixing rate of unconstrained policies can recover this
property; the explicit CATE cancellation that defines masking is
doing work that randomization between exploit and unaware cannot
replicate.



\end{document}